\documentclass[conference, 10pt]{IEEEtran}

\IEEEoverridecommandlockouts                              

\usepackage{graphicx}
\usepackage{epsfig} 
\usepackage{mathptmx}
\usepackage{amsmath} 

\usepackage{amssymb}  
\usepackage{amsthm}
\usepackage{wasysym}
\usepackage[mathcal]{euscript}
\usepackage{mathrsfs}
\usepackage{bm}
\usepackage{bbm}
\usepackage{color}
\usepackage{lipsum}
\usepackage[ruled,vlined,linesnumbered]{algorithm2e}
\usepackage[colorlinks=true,linkcolor=black,anchorcolor=black,citecolor=black,filecolor=black,menucolor=black,runcolor=black,urlcolor=black]{hyperref}
\usepackage{etoolbox}
\usepackage[table,xcdraw,dvipsnames]{xcolor}
\usepackage{multirow}
\usepackage{mathtools}
\usepackage{algpseudocode}
\usepackage{xspace}

\usepackage{outlines}
\let\labelindent\relax
\usepackage{enumitem}
\setenumerate[1]{label=\Roman*.}
\setenumerate[2]{label=\Alph*.}
\setenumerate[3]{label=\arabic*.}
\setenumerate[4]{label=\roman*.}

\usepackage{pifont}

\usepackage{booktabs}

\makeatletter
\patchcmd{\@makecaption}
  {\scshape}
  {}
  {}
  {}
\makeatother

\usepackage[
   style=ieee,
    doi=false,
    isbn=false,
    url=false,
    eprint=false,
    backend=biber,
    natbib=true
    ]{biblatex}
    
\bibliography{WaiterMotionPapers}

\title{Serving Time: Real-Time, Safe Motion Planning and Control for Manipulation of Unsecured Objects}

\author{Zachary Brei$^1$, Jonathan Michaux$^2$, Bohao Zhang$^2$, Patrick Holmes$^2$, Ram Vasudevan$^1$
\thanks{$^{1}$Mechanical Engineering, University of Michigan, Ann Arbor, MI $\langle$\texttt{breizach, ramv}$\rangle$\texttt{@umich.edu}.}
\thanks{$^{2}$Robotics, University of Michigan, Ann Arbor, MI $\langle$\texttt{jmichaux, jimzhang, pdholmes, ramv} $\rangle$\texttt{@umich.edu}.}
}

\begin{document}
\newif\ifcommentson
\commentsontrue
\newcounter{PatrickCount}
\addtocounter{PatrickCount}{1}
\newcommand{\pat}[1]{\textcolor{OliveGreen}{\ifcommentson\textbf{(\thePatrickCount) PH}: (#1)\fi}\addtocounter{PatrickCount}{1}}

\newcounter{JonCount}
\addtocounter{JonCount}{1}
\newcommand{\jon}[1]{\textcolor{Maroon}{\ifcommentson\textbf{(\theJonCount) JM}: (#1)\fi}\addtocounter{JonCount}{1}}

\newcounter{BohaoCount}
\addtocounter{BohaoCount}{1}
\newcommand{\bohao}[1]{\textcolor{Orange}{\ifcommentson\textbf{(\theBohaoCount) BZ}: (#1)\fi}\addtocounter{BohaoCount}{1}}

\newcounter{ThoughtCount}
\addtocounter{ThoughtCount}{1}
\newcommand{\thought}[1]{\textcolor{Blue}{\ifcommentson\textbf{(\theThoughtCount) Outline}: (#1)\fi}\addtocounter{ThoughtCount}{1}}

\newcounter{ShreyCount}
\addtocounter{ShreyCount}{1}
\newcommand{\shrey}[1]{{\textcolor{RoyalBlue}{\ifcommentson\textbf{(\theShreyCount) SK}: (#1)\fi}\addtocounter{ShreyCount}{1}}}

\newcounter{RamCount}
\addtocounter{RamCount}{1}
\newcommand{\Ram}[1]{\textcolor{WildStrawberry}{\ifcommentson\textbf{(\theRamCount) RV-M}: (#1)\fi}\addtocounter{RamCount}{1}}

\newcounter{ZacCount}
\addtocounter{ZacCount}{1}
\newcommand{\Zac}[1]{\textcolor{Cyan}{\ifcommentson\textbf{(\theZacCount) ZB}: (#1)\fi}\addtocounter{ZacCount}{1}}

\newcounter{XunCount}
\addtocounter{XunCount}{1}
\newcommand{\Xun}[1]{\textcolor{lime}{\ifcommentson\textbf{(\theXunCount) XF}: (#1)\fi}\addtocounter{XunCount}{1}}

\newcounter{ElenaCount}
\addtocounter{ElenaCount}{1}
\newcommand{\Elena}[1]{\textcolor{violet}{\ifcommentson\textbf{(\theElenaCount) ES}: (#1)\fi}\addtocounter{ElenaCount}{1}}

\newcounter{FixCount}
\addtocounter{FixCount}{1}
\newcommand{\fix}[1]{\textcolor{Purple}{\ifcommentson\textbf{(\theFixCount) FIX}: (#1)\fi}\addtocounter{FixCount}{1}}

\newtheorem{defn}{Definition}
\newtheorem{rem}[defn]{Remark}
\newtheorem{lem}[defn]{Lemma}
\newtheorem{prop}[defn]{Proposition}
\newtheorem{assum}[defn]{Assumption}
\newtheorem{ex}[defn]{Example}
\newtheorem{thm}[defn]{Theorem}
\newtheorem{cor}[defn]{Corollary}
\newtheorem{con}[defn]{Conjecture}
\newtheorem{problem}[defn]{Problem}

\providecommand{\R}{\ensuremath \mathbb{R}}
\providecommand{\IR}{\ensuremath \mathbb{IR}}
\providecommand{\N}{\ensuremath \mathbb{N}}
\providecommand{\Q}{\ensuremath \mathbb{Q}}
\newcommand{\unitcircle}{\mathbb{S}^1}

\newcommand{\regtext}[1]{\mathrm{\textnormal{#1}}}
\newcommand{\ol}[1]{\overline{#1}}
\newcommand{\ul}[1]{\underline{#1}}
\newcommand{\defemph}[1]{\emph{#1}}
\newcommand{\ts}[1]{\textsuperscript{#1}}

\newcommand{\comp}{^{\regtext{C}}}
\newcommand{\card}[1]{\left\vert#1\right\vert}
\newcommand{\proj}{\regtext{proj}}
\newcommand{\norm}[1]{\left\Vert#1\right\Vert}
\newcommand{\abs}[1]{\left\vert#1\right\vert}
\newcommand{\pow}[1]{\mathcal{P}\!\left(#1\right)}
\newcommand{\diag}[1]{\regtext{diag}\!\left(#1\right)}
\newcommand{\eig}[1]{\regtext{eig}\!\left(#1\right)}
\newcommand{\union}{\bigcup}
\newcommand{\intersection}{\bigcap}
\newcommand{\trans}{^\top}
\newcommand{\inv}{^{-1}}
\newcommand{\pinv}{^{\dagger}}
\newcommand{\sign}{\regtext{sign}}
\newcommand{\expm}{\regtext{exp}}
\newcommand{\logm}{\regtext{log}}
\newcommand{\skw}{_{\times}}
\newcommand{\bigO}{\mathcal{O}}
\newcommand{\bdry}[1]{\regtext{bd}\!\left(#1\right)}
\renewcommand{\ker}[1]{\regtext{ker}\!\left(#1\right)}
\newcommand{\convhull}[1]{\regtext{CH}\!\left(#1\right)}

\newcommand{\lbl}[1]{_{\regtext{#1}}}
\newcommand{\lo}{\lbl{lo}}
\newcommand{\hi}{\lbl{hi}}

\newcommand{\emptyarr}{[\ ]}
\newcommand{\zeros}{\textit{0}}
\newcommand{\ones}{\textit{1}}
\newcommand{\eye}{\regtext{\textit{I}}}


\newcommand{\interval}[1]{[ #1 ]}
\newcommand{\iv}[1]{[ #1 ]}
\newcommand{\nom}[1]{#1}
\newcommand{\pz}[1]{\mathbf{#1}}
\newcommand{\pzgreek}[1]{\bm{#1}}
\newcommand{\PZ}[1]{\mathcal{PZ}\left(#1\right)}
\newcommand{\pzk}[1]{\pz{ #1 ; k }}
\newcommand{\pzi}[1]{\pz{ #1 }(\pz{T_i};\pz{K})}
\newcommand{\pzki}[1]{\pz{ #1 }(\pz{T_i};k)}
\newcommand{\setop}[1]{{\mathrm{\texttt{#1}}}}
\newcommand{\lb}[1]{\underline{#1}}
\newcommand{\ub}[1]{\overline{#1}}


\newcommand{\qjnot}{q_j}
\newcommand{\qdjnot}{\dot{q}_{j}}
\newcommand{\qddjnot}{\ddot{q}_{j}}
\newcommand{\qajnot}{q_{a, j}}
\newcommand{\qajdotnot}{\dot{q}_{a, j}}
\newcommand{\qajddotnot}{\ddot{q}_{a, j}}

\newcommand{\pzqi}{\pzi{q}}
\newcommand{\pzqAi}{\pzi{\qA}}
\newcommand{\pzqdi}{\pzi{\dot{q}}}
\newcommand{\pzqdai}{\pzi{\dot{q}_{a}}}
\newcommand{\pzqddi}{\pzi{\ddot{q}}}
\newcommand{\pzqddai}{\pzi{\ddot{q}_{a}}}
\newcommand{\pzqdesi}{\pzi{q_{d}}}
\newcommand{\pzqddesi}{\pzi{\dot{q}_{d}}}
\newcommand{\pzqdddesi}{\pzi{\ddot{q}_{d}}}
\newcommand{\pzqdeski}{\pzki{q_{d}}}
\newcommand{\pzqddeski}{\pzki{\dot{q}_{d}}}
\newcommand{\pzqdddeski}{\pzki{\ddot{q}_{d}}}
\newcommand{\pzui}{\pzki{u}}
\newcommand{\pzqki}{\pzki{q}}
\newcommand{\pzqAki}{\pzki{\qA}}
\newcommand{\pzqdki}{\pzki{\dot{q}}}
\newcommand{\pzuki}{\pzki{u}}

\newcommand{\pzqji}{\pzi{q_j}}
\newcommand{\pzqAji}{\pzi{\qA_j}}
\newcommand{\pzqli}{\pzi{q_l}}
\newcommand{\pzqdji}{\pzi{\dot{q}_j}}
\newcommand{\pzqdaji}{\pzi{\dot{q}_{a,j}}}
\newcommand{\pzqddji}{\pzi{\ddot{q}_j}}
\newcommand{\pzqddaji}{\pzi{\ddot{q}_{a,j}}}
\newcommand{\pzqdesji}{\pzi{q_{d,j}}}
\newcommand{\pzqddesji}{\pzi{\dot{q}_{d,j}}}
\newcommand{\pzqdddesji}{\pzi{\ddot{q}_{d,j}}}
\newcommand{\pzqdesjki}{\pzki{q_{d,j}}}
\newcommand{\pzqddesjki}{\pzki{\dot{q}_{d,j}}}
\newcommand{\pzqdddesjki}{\pzki{\ddot{q}_{d,j}}}
\newcommand{\pzuji}{\pzki{u_j}}
\newcommand{\pzqjki}{\pzki{q_j}}
\newcommand{\pzqdjki}{\pzki{\dot{q}_j}}
\newcommand{\pzujKi}{\pz{u}(\pzqAi, \nomparams, \intparams)}
\newcommand{\pzujki}{\pz{u}(\pzqAki, \nomparams, \intparams)}
\newcommand{\pzFKjki}{\pz{FK_j}(\pzqki)}
\newcommand{\pzFOjki}{\pz{FO_j}(\pzqki)}
\newcommand{\pzFKjKi}{\pz{FK_j}(\pzqi)}
\newcommand{\pzFOjKi}{\pz{FO_j}(\pzqi)}

\newcommand{\pzpboundj}{\bm{\epsilon}_{\mathbf{p, j}}}
\newcommand{\pzvboundj}{\bm{\epsilon}_{\mathbf{v, j}}}

\newcommand{\pzg}{g}
\newcommand{\pzv}{x}
\newcommand{\pze}{\alpha}
\newcommand{\pzn}{{n_g}}
\newcommand{\pzgi}{g_i}
\newcommand{\pzei}{\alpha_i}

\newcommand{\q}{q(t)}
\newcommand{\qd}{\dot{q}(t)}
\newcommand{\qdd}{\ddot{q}(t)}
\newcommand{\qa}{q_a(t)}
\newcommand{\qadot}{\dot{q}_a(t)}
\newcommand{\qaddot}{\ddot{q}_a(t)}
\newcommand{\qak}{q_a(t; k)}
\newcommand{\qakdot}{\dot{q}_a(t; k)}
\newcommand{\qakddot}{\ddot{q}_a(t; k)}
\newcommand{\qdes}{q_d(t)}
\newcommand{\qdesdot}{\dot{q}_d(t)}
\newcommand{\qdesddot}{\ddot{q}_d(t)}
\newcommand{\qdesk}{q_d(t; k)}
\newcommand{\qdeskdot}{\dot{q}_d(t; k)}
\newcommand{\qdeskddot}{\ddot{q}_d(t; k)}


\newcommand{\qj}{q_j(t)}
\newcommand{\ql}{q_l(t)}
\newcommand{\qdj}{\dot{q}_{j}(t)}
\newcommand{\qddj}{\ddot{q}_{j}(t)}
\newcommand{\qaj}{q_{a, j}(t)}
\newcommand{\qajdot}{\dot{q}_{a, j}(t)}
\newcommand{\qajddot}{\ddot{q}_{a, j}(t)}
\newcommand{\qdesj}{q_{d, j}(t)}
\newcommand{\qdesjdot}{\dot{q}_{d, j}(t)}
\newcommand{\qdesjddot}{\ddot{q}_{d, j}(t)}
\newcommand{\qdeskj}{q_{d, j}(t; k)}
\newcommand{\qdeskjdot}{\dot{q}_{d, j}(t; k)}
\newcommand{\qdeskjddot}{\ddot{q}_{d, j}(t; k)}

\newcommand{\nq}{n_q}
\newcommand{\nne}{n_e}
\newcommand{\nt}{n_t}
\newcommand{\nf}{n_f}
\newcommand{\Nq}{ N_q }
\newcommand{\Ne}{ N_e }
\newcommand{\Nt}{ N_t }

\newcommand{\err}{e}
\newcommand{\errdot}{\dot{e}}
\newcommand{\errdotj}{\dot{e}_j}
\newcommand{\errj}{e_j}
\newcommand{\errjdot}{\dot{e}_j}
\newcommand{\errddot}{\ddot{e}}
\newcommand{\roblyap}{V(\qA(t),\Delta)}
\newcommand{\robh}{h(\qA(t),\Delta)}
\newcommand{\robv}{v}
\newcommand{\robr}{r}
\newcommand{\robrj}{r_j}
\newcommand{\robw}{w}
\newcommand{\robrdot}{\dot{r}}
\newcommand{\roblyapmax}{\overline{V}(q, r)}
\newcommand{\roblyapdot}{\dot{V}(\qA(t))}

\newcommand{\robH}{H}
\newcommand{\robhmin}{\underline{h}(\qA(t),[\Delta])}
\newcommand{\robhdot}{\dot{h}(\qA(t))}
\newcommand{\roblevel}{V_M}
\newcommand{\robcoeff}{\gamma}
\newcommand{\robKinf}{\alpha}
\newcommand{\robgain}{\alpha_c}
\newcommand{\ultbound}{\sqrt{\frac{2 \roblevel}{\sigma_m}}}
\newcommand{\pbound}{\epsilon_p}
\newcommand{\vbound}{\epsilon_v}
\newcommand{\pboundj}{\epsilon_{p, j}}
\newcommand{\vboundj}{\epsilon_{v, j}}
\newcommand{\pboundvec}{E_p}
\newcommand{\vboundvec}{E_v}
\newcommand{\epvar}{x_{e_p}}
\newcommand{\evvar}{x_{e_v}}
\newcommand{\epvarj}{x_{e_{p, j}}}
\newcommand{\evvarj}{x_{e_{v, j}}}
\newcommand{\qA}{q_A}

\providecommand{\R}{\ensuremath \mathbb{R}}
\newcommand{\plan}{_p}
\newcommand{\prev}{\lbl{prev}}
\providecommand{\tfin}{t\lbl{f}}

\newcommand{\zi}{z_i}
\newcommand{\zj}{z_j}
\newcommand{\rbf}{\mathbf{r}(t)}

\newcommand{\bM}{M}
\newcommand{\Mq}{M(\q, \Delta_{arm})}
\newcommand{\Mqdot}{\dot{M}(q(t), \Delta)}
\newcommand{\bMt}{\Tilde{M}}

\newcommand{\bC}{C}
\newcommand{\Cq}{C(\q, \qd)}
\newcommand{\Cqd}{C(\q, \qd, \Delta_{arm})}
\newcommand{\bCt}{\Tilde{C}}

\newcommand{\bG}{G}
\newcommand{\Gq}{G(\q)}
\newcommand{\Gqd}{G(\q, \Delta_{arm})}
\newcommand{\bGt}{\Tilde{G}}

\newcommand{\bfc}{\mathbf{c}}
\newcommand{\bfI}{\mathbf{I}}

\newcommand{\bH}{H}
\newcommand{\Hq}{H(\q, \Delta)}
\newcommand{\btau}{\tau}
\newcommand{\tauqd}{\tau(\q, \qd, \qdd, \Delta)}
\newcommand{\tauqdcontact}{\tau_{\contact}(\q, \qd, \qdd, \Delta)}

\newcommand{\intparams}{[\Delta]}
\newcommand{\intparamsdb}{[\Delta]\lbl{db}}
\newcommand{\nomparams}{\Delta_0}
\newcommand{\trueparams}{\Delta}
\newcommand{\pzparams}{\pzgreek{\Delta}}

\makeatletter
\newcommand{\smalloplus}{\mathbin{\mathpalette\make@small\oplus}}
\newcommand{\smallotimes}{\mathbin{\mathpalette\make@small\otimes}}


\newcommand{\nominal}{\tau(t) = \bM(q(t), \nomparams) \ddot{q}_a(t) + \bC(q(t), \dot{q}(t), \nomparams) \dot{q}_a(t) + \bG(q(t), \nomparams)}

\newcommand{\robust}{ v = (\kappa(t) + \|\rho([\Phi]) \| + \varphi(t)) r}

\newcommand{\controller}{ u(t;k) = \tau(t;k) - \robv(t;k)}

\newcommand{\bPhi}{\Phi}

\newcommand{\wdistlong}{w(\Delta)}
\newcommand{\wdist}{w}
\newcommand{\wdistlongi}{w_i(\Delta)}

\newcommand{\wdistinterval}{w}


\newcommand{\closedloop}{\bH(\q,\Delta)\dot{r} + \bC(\q, \qd, \Delta)r  = -v + \wdist }

\newcommand{\lambdamin}{\lambda_m}
\newcommand{\lambdamax}{\lambda_M}
\newcommand{\sigm}{\sigma_{m}}
\newcommand{\sigM}{\sigma_{M}}

\newcommand{\Hquad}{\hspace{0.5em}}

\newcommand{\FK}{\regtext{\small{FK}}}
\newcommand{\IK}{\regtext{\small{IK}}}
\newcommand{\FO}{\regtext{\small{FO}}}
\newcommand{\IO}{\regtext{\small{IO}}}
\newcommand{\FS}{\regtext{\small{FS}}}
\newcommand{\IS}{\regtext{\small{IS}}}
\newcommand{\FC}{\regtext{\small{FC}}}
\newcommand{\IC}{\regtext{\small{IC}}}
\newcommand{\ID}{\regtext{\small{ID}}}

\newcommand{\exact}{^{\regtext{exact}}}
\newcommand{\slice}{\textnormal{\texttt{slice}}}
\newcommand{\eval}{\textnormal{\texttt{eval}}}
\newcommand{\stack}{\textnormal{\texttt{stack}}}
\newcommand{\reduce}{\textnormal{\texttt{reduce}}}
\newcommand{\getCoeffValue}{\texttt{getCoeffValue}}
\newcommand{\timeint}{([0, T])}

\newcommand{\Aobs}{A_O}
\newcommand{\bobs}{b_O}
\newcommand{\hobs}{h\lbl{obs}}
\newcommand{\nObs}{n_\mathscr{O}}
\newcommand{\obsset}{\mathscr{O}}

\newcommand{\SO}{\regtext{\small{SO}}}

\newcommand{\kj}{k_j}
\newcommand{\Kj}{K_j}

\newcommand{\kscale}{\eta_1}
\newcommand{\kjscale}{\eta_{j, 1}}
\newcommand{\koffset}{\eta_2}
\newcommand{\kjoffset}{\eta_{j, 2}}
\newcommand{\kvar}{x_k}
\newcommand{\kjvar}{x_{k_j}}

\newcommand{\tvar}{x_t}
\newcommand{\tvari}{x_{t_{i}}}

\newcommand{\F}{\mathcal{F}}

\newcommand{\eh}{\hat{e}}

\newcommand{\tsum}{{\textstyle\sum}}

\newcommand{\ujt}{u_j(t)}


\newcommand{\pjt}{p_j(t)}
\newcommand{\Rjt}{R_j(t)}










\newcommand{\unsafeobs}{\lbl{obs}}
\newcommand{\unsafeself}{\lbl{self}}
\newcommand{\unsafejoint}{\lbl{lim}}
\newcommand{\jlim}{\lbl{lim}}
\newcommand{\selfidx}{I\self}

\newcommand{\qlim}{q_{j,\regtext{lim}}}
\newcommand{\dqlim}{\dot{q}_{j,\regtext{lim}}}
\newcommand{\ddqlim}{\ddot{q}_{j,\regtext{lim}}}
\newcommand{\ulim}{u_{j,\regtext{lim}}}

\newcommand{\hitj}{h_i^{t, j}}
\newcommand{\buf}{\lbl{buf}}
\newcommand{\slc}{\lbl{slc}}
\newcommand{\Aitj}{A_i^{t, j}}
\newcommand{\bitj}{b_i^{t, j}}
\newcommand{\hitself}{h_{i_1, i_2}^{t}}
\newcommand{\Aitself}{A_{i_1, i_2}^{t}}
\newcommand{\bitself}{b_{i_1, i_2}^{t}}

\newcommand{\hqim}{h_{q_i^-}}
\newcommand{\hqip}{h_{q_i^+}}
\newcommand{\hdqim}{h_{\dot{q}_i^-}}
\newcommand{\hdqip}{h_{\dot{q}_i^+}}
\newcommand{\hijoint}{h_{i, \regtext{lim}}}











\newcommand{\initq}{q_{d_0}}
\newcommand{\initqj}{q_{d, j_{0}}}
\newcommand{\initdq}{\dot{q}_{d_0}}
\newcommand{\initdqj}{\dot{q}_{d, j_{0}}}
\newcommand{\initddq}{\ddot{q}_{d_0}}
\newcommand{\initddqj}{\ddot{q}_{d, j_{0}}}

\newcommand{\costfunc}{\phi}




\newcommand{\iss}{_{i}^{i}}
\newcommand{\issm}{_{i-1}^{i}}
\newcommand{\issmu}{_{i}^{i-1}}
\newcommand{\issmi}{_{i-1, i}^{i}}
\newcommand{\issp}{_{i+1}^{i}}
\newcommand{\issmm}{_{i-1}^{i-1}}
\newcommand{\isspp}{_{i+1}^{i+1}}
\newcommand{\issa}{_{a, i}^{i}}
\newcommand{\issc}{_{c, i}^{i}}
\newcommand{\issma}{_{a, i-1}^{i}}
\newcommand{\isspa}{_{a, i+1}^{i}}
\newcommand{\issmma}{_{a, i-1}^{i-1}}
\newcommand{\issppa}{_{a, i+1}^{i+1}}

\newcommand{\jss}{_{j}^{j}}
\newcommand{\jssm}{_{j-1}^{j}}
\newcommand{\jssmu}{_{j}^{j-1}}
\newcommand{\jssmj}{_{j-1, j}^{j}}
\newcommand{\jssp}{_{j+1}^{j}}
\newcommand{\jssmm}{_{j-1}^{j-1}}
\newcommand{\jsspp}{_{j+1}^{j+1}}
\newcommand{\jssa}{_{a, j}^{j}}
\newcommand{\jssCOM}{_{CoM, j}^{j}}
\newcommand{\jssc}{_{c, j}^{j}}
\newcommand{\jssma}{_{a, j-1}^{j}}
\newcommand{\jsspa}{_{a, j+1}^{j}}
\newcommand{\jssmma}{_{a, j-1}^{j-1}}
\newcommand{\jssppa}{_{a, j+1}^{j+1}}

\newcommand{\lssmu}{_{l}^{l-1}}

\newcommand{\qgoal}{q\lbl{goal}}
\newcommand{\qstart}{q\lbl{start}}

\newcommand{\MethodName}{WAITR}

\newcommand{\normrho}{||\rho([\Phi])||}
\newcommand{\normwmax}{||w_M||}
\newcommand{\wmax}{w_M}
\newcommand{\wmaxj}{w_{M, j}}

\newcommand{\timestep}{\Delta t}


\newcommand{\objrad}{r}
\newcommand{\ZMPpoint}{Z}
\newcommand{\heightobjCOM}{h_{CoM}}
\newcommand{\CoefFric}{\mu_{s}} 
\newcommand{\eenorm}{\hat{n}} 
\newcommand{\MGravInert}{\tilde{{\text{n}}}} 
\newcommand{\GenMGravInert}{\tilde{\text{n}}}
\newcommand{\FGravInert}{\tilde{{\text{f}}}} 
\newcommand{\GenFGravInert}{\tilde{\text{f}}}
\newcommand{\MGravInertProj}{\tilde{{\text{n}}}_{\COMproj}} 
\newcommand{\FGravInertProj}{\tilde{{\text{f}}}_{\COMproj}} 
\newcommand{\ZMPvec}{p_\text{ZMP}}
\newcommand{\COMproj}{p_{\text{CoM}}}
\newcommand{\objCOM}{G}

\newcommand{\numop}[1]{{\mathrm{\textnormal{\texttt{#1}}}}}

\newcommand{\elapsed}{\lbl{e}}
\providecommand{\tfin}{t\lbl{f}}

\newcommand{\homtrans}{S}


\newcommand{\COM}{CoM}
\newcommand{\COMvel}{\text{v}_{j,\COM}^{j}}
\newcommand{\COMaccel}{\dot{\text{v}}_{j,\COM}^{j}}
\newcommand{\angvel}{\omega\jss}
\newcommand{\angaccel}{\dot{\omega}\jss}

\newcommand{\qT}{q_D}
\newcommand{\pzkiqT}{\pz{\qT}(\pz{T_i};k,\intparams)}

\newcommand{\contact}{o}
\newcommand{\contactObj}{c}

\newcommand{\sep}{h_{sep}} 
\newcommand{\slip}{h_{slip}}
\newcommand{\tip}{h_{tip}}

\newcommand{\relativeMotion}{\text{relative motion}}
\newcommand{\RelativeMotion}{\text{Relative Motion}}
\newcommand{\itrelativeMotion}{relative motion}

\newcommand{\genwrench}[1]{\mathrm{w}_{#1}} 
\newcommand{\genwrenchPZ}[1]{\mathbf{w}_{#1}}
\newcommand{\genwrenchext}[1]{{\mathrm{F}}_{#1,ext}}
\newcommand{\genwrenchextPZ}[1]{\mathbf{\mathcal{F}}_{#1,ext}}
\newcommand{\cwrench}{\mathrm{w}_{\contact}}
\newcommand{\cwrenchObj}{\mathrm{w}_{\contactObj}}
\newcommand{\cwrenchPZ}{\mathbf{w}_{\contact}}
\newcommand{\gravinertwrench}[1]{\tilde{\mathrm{w}}_{#1}} 
\newcommand{\gravinertwrenchPZ}[1]{\tilde{\mathbf{w}}_{#1}(t)} 
\newcommand{\genjointforce}{\text{f}\jss}
\newcommand{\genjointforcenot}{\text{f}\jss}
\newcommand{\genjointforcenotPZ}{\text{\mathbf{f}}\jss}
\newcommand{\genjointtorque}{\text{n}\jss}
\newcommand{\genjointtorquenot}{\text{n}\jss}
\newcommand{\genCOMforce}{\text{F}\jss}
\newcommand{\genCOMforcePZ}{\mathbf{F}\jss}
\newcommand{\genCOMtorque}{\text{N}\jss}
\newcommand{\genCOMtorquePZ}{\mathbf{N}\jss}

\newcommand{\cforce}{\text{f}_{\contact}}
\newcommand{\cforcetan}{{\text{f}}_{\contact,T}}
\newcommand{\cforcenormal}{{\text{f}}_{\contact,N}}
\newcommand{\cxforce}{\text{f}_{\contact,x}}
\newcommand{\cyforce}{\text{f}_{\contact,y}}
\newcommand{\czforce}{\text{f}_{\contact,z}}

\newcommand{\cmoment}{{\text{n}}_{\contact}}
\newcommand{\cxmoment}{\text{n}_{\contact,x}}
\newcommand{\cymoment}{\text{n}_{\contact,y}}
\newcommand{\czmoment}{\text{n}_{\contact,z}}
\newcommand{\czmomentmax}{\text{n}_{\contact,z,max}}

\newcommand{\cforcenot}{{\text{f}}_{\contact}}
\newcommand{\cforcenotPZ}{\text{\mathbf{f}}_{\contact}}
\newcommand{\cforcetannot}{{\text{f}}_{\contact,T}}
\newcommand{\cforcetannotPZ}{\text{\mathbf{f}}_{\contact,T}}
\newcommand{\cforcenormalnot}{{\text{f}}_{\contact,N}}
\newcommand{\cforcenormalnotPZ}{\text{\mathbf{f}}_{\contact,N}}
\newcommand{\cxforcenot}{\text{f}_{\contact,x}}
\newcommand{\cxforcenotPZ}{\text{\mathbf{f}}_{\contact,x}}
\newcommand{\cyforcenot}{\text{f}_{\contact,y}}
\newcommand{\cyforcenotPZ}{\text{\mathbf{f}}_{\contact,y}}
\newcommand{\czforcenot}{\text{f}_{\contact,z}}
\newcommand{\czforcenotPZ}{\text{\mathbf{f}}_{\contact,z}}

\newcommand{\cmomentnot}{{\text{n}}_{\contact}}
\newcommand{\cmomentnotPZ}{\text{\mathbf{n}}_{\contact}}
\newcommand{\cxmomentnot}{\text{n}_{\contact,x}}
\newcommand{\cxmomentnotPZ}{\text{\mathbf{n}}_{\contact,x}}
\newcommand{\cymomentnot}{\text{n}_{\contact,y}}
\newcommand{\cymomentnotPZ}{\text{\mathbf{n}}_{\contact,y}}
\newcommand{\czmomentnot}{\text{n}_{\contact,z}}
\newcommand{\czmomentnotPZ}{\text{\mathbf{n}}_{\contact,z}}
\newcommand{\czmomentmaxnot}{\text{n}_{\contact,z,max}}

\newcommand{\TROtitle}{ARMOUR}

\maketitle

\begin{abstract}
    \label{sec:abstract}
    A key challenge to ensuring the rapid transition of robotic systems from the industrial sector to more ubiquitous applications is the development of algorithms that can guarantee safe operation while in close proximity to humans.
    Motion planning and control methods, for instance, must be able to certify safety while operating in real-time in arbitrary environments and in the presence of model uncertainty.
    This paper proposes Wrench Analysis for Inertial Transport using Reachability (\MethodName{}), a certifiably safe motion planning and control framework for serial link manipulators that manipulate unsecured objects in arbitrary environments.
    \MethodName{} uses reachability analysis to construct over-approximations of the contact wrench applied to unsecured objects, which captures uncertainty in the manipulator dynamics, the object dynamics, and contact parameters such as the coefficient of friction.
    An optimization problem formulation is presented that can be solved in real-time to generate provably-safe motions for manipulating the unsecured objects.
    This paper illustrates that \MethodName{} outperforms state of the art methods in a variety of simulation experiments and demonstrates its performance in the real-world.
\end{abstract}
\section{Introduction}
\label{sec:intro_old}


A key challenge of non-prehensile robotic manipulation is safe trajectory planning for manipulation of supported objects. 
The lack of force/form closure in non-prehensile manipulation tasks means that an incorrectly applied wrench can result in damage to the unsecured objects or the environment.
This challenge is compounded when the non-prehensile manipulation task must be performed in environments that may only be known at run-time as this requires real-time generation of motion plans.
Because the model of the object being manipulated may not be known perfectly, these methods must be able to account for model uncertainty as well.
Therefore, robotic manipulators need to be capable of finding and robustly executing provably safe trajectories in real-time.

\begin{figure}
    \centering
    \includegraphics[width=1\columnwidth]{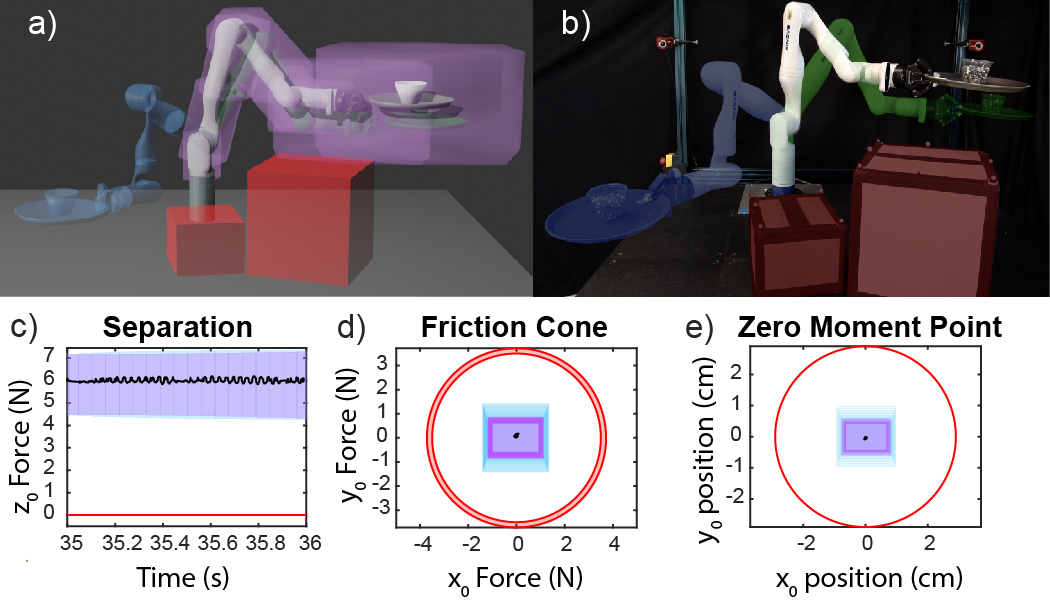}
    \caption{
    This paper considers the problem of safe motion planning for manipulation of unsecured objects with uncertain dynamics such as manipulating an unsecured cup filled with an uncertain mass around randomly placed obstacles (red) such that the cup does not move relative to the tray supporting it.
    \MethodName{} operates in receding-horizon fashion, moving from a start configuration (blue) to a global goal (green) by repeatedly generating new motion plans in real-time.
    In each motion planning iteration, \MethodName{} calculates a reachable set (blue and purple) for the contact wrench between the manipulator and the object as well as a Forward Reachable Set (FRS) for the whole manipulator system for a continuum of possible motion plans.
    The FRS is shown in purple in a) for a single planning iteration.
    \MethodName{} solves a constrained trajectory optimization problem to find a collision-free motion in this FRS that does not result in relative motion while making progress towards an intermediate waypoint (grey) and the global goal.
    Parts c)-e) show the contact constraints enforced during a hardware experiment for a single planning iteration.
    \vspace*{-0.5cm}
    }
    \label{fig:overview}
\end{figure}

To safely manipulate supported objects, the \relativeMotion{} between the objects and supporting surface needs to be constrained.
To generate such a motion, previous methods formulate optimal control problems which solve for a time-optimal trajectory along a pre-specified, discretized geometric path that is assumed to be collision-free \cite{Debrouwere2013,Csorvasi2017,Gattringer2022,Gattringer2022a}.
Relative motion is partially or fully constrained through a friction cone constraint \cite{Csorvasi2017}, as well as a lift and tipping constraint \cite{Debrouwere2013} and a rotational friction cone constraint \cite{Gattringer2022,Gattringer2022a}.
Note that relative motion is fully constrained only if all six degrees of freedom of the object are constrained.
Other methods also incorporate optimization of the geometric path into the time-optimal trajectory following task \cite{Flores,Ichnowski2021,Avigal2022,Gattringer2022}. 



Unfortunately, these methods are unable to generate a provably safe continuous-time trajectory in real-time while also accounting for modeling error.
In particular, during implementation, the optimization problems are only able to represent the collision avoidance constraint at discrete time-instances.
Finer discretizations result in trajectories that are more likely to be safe and dynamically feasible, but have higher computation times.
More troublingly, time-optimal trajectories commonly result in the robot operating near the constraint boundaries.
If not dealt with robustly, then uncertainty in the properties of the manipulator or object being manipulated could result in the \emph{executed} time-optimal trajectory being unsafe.
Uncertainty in both parameter estimation and execution error was considered in \cite{Luo2017}, however, the contact and dynamics constraints are still enforced at discrete time instances.

The contributions of this paper are two-fold.
First, we develop a framework that uses polynomial zonotopes to represent the reachability of wrenches exerted throughout a serial chain manipulator, including contact wrenches applied to manipulated objects.
Second, we formulate a trajectory optimization problem that can be tractably solved in real-time with continuous-time safety guarantees for preventing relative motion of unsecured objects.
As illustrated in Fig. \ref{fig:overview}, this framework is implemented in a receding horizon fashion and can robustly handle uncertainty in both the manipulator and object parameters.
Note this framework extends \TROtitle{} \cite{michaux2023}, which provides collision-free and dynamically feasible trajectories in real-time that account for tracking error due to modeling uncertainty in the manipulator, but is unable to make guarantees about the wrenches applied during motion. 
To make this distinction clear, \MethodName{} is tested in simulation and on hardware and compared against \TROtitle{}.

Next, we briefly summarize the structure of this paper.
Sec. \ref{sec:preliminaries} presents relevant notation and mathematical objects.
Sec. \ref{sec:modeling} describes the robot dynamics, the contact model and constraints, and presents a continuous-time optimization problem that ensures safe manipulation of unsecured objects. 
Sec. \ref{sec:formulation} summarizes a passivity-based robust controller \cite{michaux2023}, then describes the generation of polynomial zonotope overapproximations of the manipulator's trajectory and contact wrench and how that is used to form a tractable implementation of the continuous-time optimization problem which has continuous-time safety guarantees.
Section \ref{sec:experimental} details the simulation and hardware experiments. 
\vspace*{-0.15cm}
\section{Preliminaries}
\label{sec:preliminaries}

This section describes our notation conventions, several set representations, and operations on these set representation.
These operations are summarized for convenience in Tab. \ref{tab:poly_zono_operations}.
This paper uses a letter preceding an equation number, e.g. (A12), to refer to equations provided in supplementary appendices which can be found at \url{https://roahmlab.github.io/waitr-dev/}. 


The $n$-dimensional real numbers are $\R^n$, natural numbers $\N$, the unit circle is $\mathbb{S}^1$, and the set of $3 \times 3$ orthogonal matrices is $\SO(3)$.
Subscripts may index elements of a vector or a set.
Let $U$, $V \subset \R^n$.
For a point $u \in U$, $\{ u \} \subset U$ is the set containing only $u$.
The \emph{power set} of $U$ is $\pow{U}$.
The \emph{Minkowski Sum} is $U \oplus V$
the \emph{Minkowski Difference} is $U \ominus V = U \oplus (-V)$.
For vectors $a, b \in \R^3$, we write the cross product $a \times b$.
If $n = 3$, the \emph{set-based cross product} is defined as $U \otimes V = \{ u \times v \, \mid \, u \in U, v \in V \}$.
If $\{U_i \subset \R^n\}_{i=1}^{m}$ then let $\bigtimes_{i=1}^n U_i$ denote the Cartesian product of the $U_i$'s.
Let $X \subset \R^{n \times n}$ be a set of square matrices.
Then, \emph{set-based multiplication} is defined as $XV = \{ Av, \, \mid \, A \in X, v \in V \subseteq \R^n \}$.
Let $\zeros$ (resp. $\ones$) denote a matrix of zeros (resp. ones) of appropriate size, and let $\eye_n$ be the $n \times n$ identity matrix.


\begin{table}[t]
    \centering
    \begin{tabular}{c|c}
        Operation & Computation \\
        \hline
        $\iv{z} \to \pz{z}$ (Interval Conversion) (A8) & Exact \\ 
        $\pz{P}_1 \oplus \pz{P}_2$ (PZ Minkowski Sum) (A10) & Exact \\ 
        $\pz{P}_1\pz{P}_2$ (PZ Multiplication) (A12) & Exact \\ 
        $\pz{P}_1 \otimes \pz{P}_2$ (PZ Cross Product) (A13) & Exact \\ 
        $\setop{slice}(\pz{P}, \pzv_j, \sigma)$ (A14) & Exact \\ 
       $\setop{inf}(\pz{P})$ (A16) and $\setop{sup}(\pz{P})$ (A15) & Overapproximative \\ 
        $f(\pz{P}_1) \subseteq \pz{P}_2$ (Taylor expansion) A(23) & Overapproximative 
    \end{tabular}
    \caption{Summary of polynomial zonotope operations.
    \vspace*{-0.85cm}}
    \label{tab:poly_zono_operations}
\end{table}

We use intervals to describe uncertain manipulator parameters \cite{hickey2001interval}.
An \emph{n-dimensional interval} is a set $\interval{x} = \{ y \in \R^n \, \mid \, \lb{x}_i \leq y_i \leq \ub{x}_i, \Hquad \forall i = 1, \dots, n\}$.
When the bounds are important, we denote an interval $\interval{x}$ by $\interval{\lb{x}, \ub{x}}$, where $\lb{x}$ and $\ub{x}$ are the infimum and supremum, respectively, of $\interval{x}$.
Let $\mathbb{I}\mathbb{R}^{n}$ be the set of all real-valued $n$-dimensional interval vectors. 

Because multidimensional intervals can only describe hyperrectangles, they may be overly conservative when outer approximating other shapes.
To build a tighter outer-approximative representation of multidimensional sets,  we use zonotopes and polynomial zonotopes.
Here, we provide necessary definitions of zonotopes and polynomial zonotopes as well as a few important operations, with more operations summarized in Tab. \ref{tab:poly_zono_operations}. 
App. A \footnote[1]{Supplementary Appendices found at https://roahmlab.github.io/waitr-dev/} rigorously defines the operations given in Tab. \ref{tab:poly_zono_operations}.

To define zonotopes, we introduce an indeterminate vector $\pzv \in [-1, 1]^\pzn$ and exponents $\pzei \in \N^\pzn$.
Letting $\pze_1 = [1, 0, \ldots, 0]$, $\pze_2 = [0, 1, 0, \ldots, 0]$, \ldots, $\pze_\pzn = [0, 0, \ldots, 0, 1]$, we note that $\pzv^{\pze_1} = \pzv_1, \ldots ,\pzv^{\pze_\pzn} = \pzv_\pzn$ where the exponentiation is applied element-wise.
A \emph{zonotope} $Z \subset \R^n$  is a convex, centrally-symmetric polytope defined by a \emph{center} $g_0 \in \R^n$ and \emph{generators} $g_i \in \R^n$ as
    $Z = \{
            z \in \R^n \, \mid \,
            z = \sum_{i=0}^{\pzn} \pzgi \pzv ^{\pzei}, \, \pzv \in [-1, 1]^\pzn
        \},$
where there are $\pzn \in \N$ generators.

Note we have written $Z$ as the set of points produced by the polynomial $p(\pzv) = \sum_{i=0}^{\pzn} \pzgi \pzv ^{\pzei}$ over the domain $\pzv \in [-1, 1]^\pzn$.
A polynomial zonotope is the more general case where exponents $\pzei \in \N^\pzn$ can be arbitrary, which includes zonotopes as a subset.
When we need to emphasize the generators and exponents of a polynomial zonotope, we write $\pz{P} = \PZ{\pzgi, \pzei, \pzv}$.
Note we exclusively use bold symbols to denote polynomial zonotopes.

We use polynomial zonotopes $\pz{P}$ to represent a set of possible positions of a robot arm operating near an obstacle.
It may be beneficial to know whether a particular choice of $\pz{P}$'s indeterminates yields a subset of positions that could collide with the obstacle.
To this end, we introduce the operation of ``slicing'' a polynomial zonotope $\pz{P} = \PZ{ \pzgi, \pzei, \pzv }$  by evaluating an element of the indeterminate $\pzv$.
Given the $j$\ts{th} indeterminate $\pzv_j$ and a value $\sigma \in [-1, 1]$, 
let $\setop{slice}(\pz{P}, \pzv_j, \sigma)$ denote the slicing operation which  yields a subset of $\pz{P}$ by plugging $\sigma$ into the specified element $\pzv_j$ and is formally defined in (A14).
It is possible to efficiently generate upper and lower bounds on the values of a polynomial zonotope through overapproximation.
In particular, we define the $\setop{sup}$ and $\setop{inf}$ operations which return these upper and lower bounds, respectively in (A15) and (A16).
Note that for any $z \in \pz{P}$,  $\setop{sup}(\pz{P}) \geq z$ and $\setop{inf}(\pz{P}) \leq z$, where the inequalities are taken element-wise.



\section{Wrench and Online Optimization} 
\label{sec:modeling}
%
This section introduces an extended arm model for the manipulator-tray-object system and a contact model between the tray and object. 
To simplify exposition, this paper primarily focuses on ensuring that the manipulator operates without causing any relative motion between the serving tray and the objects on it.
Section \ref{sec:experimental} describes how to apply \TROtitle{} \cite{michaux2023} to ensure that the robot does not collide with obstacles while satisfying input and joint limit constraints.

\subsection{Manipulator Model}

Given an $\nq$-dimensional serial robotic manipulator with configuration space $Q$ and a compact time interval $T \subset \R$ we define a trajectory for the configuration as $q: T \to Q \subset \R^{\nq}$. 
The velocity of this configuration trajectory is $\dot{q}: T \to \R^{\nq}$.
We make the following assumptions about the robot model:
\begin{assum}
The robot operates in a three dimensional workspace, $W \subset \R^3$.
The robot is composed of revolute joints, where the $j$\ts{th} joint actuates the robot's $j$\ts{th} link.
 The robot has encoders that allow it to measure its joint positions and velocities.
The robot is fully actuated, where the robot's input $u: T \to \R^{\nq}$.
\end{assum}
\noindent We make the one-joint-per-link assumption with no loss of generality because joints with multiple degrees of freedom (\emph{e.g.}, spherical joints) may be represented using links with zero length.
This work can be extended to robots with prismatic joints by straightforward extensions to the RNEA algorithms presented in \cite{michaux2023}.
Finally, let $\Nq = \{1,\ldots,\nq\}$.

\subsubsection{Tray and Object Model}

\begin{figure}
    \centering
    \includegraphics[width=1\columnwidth]{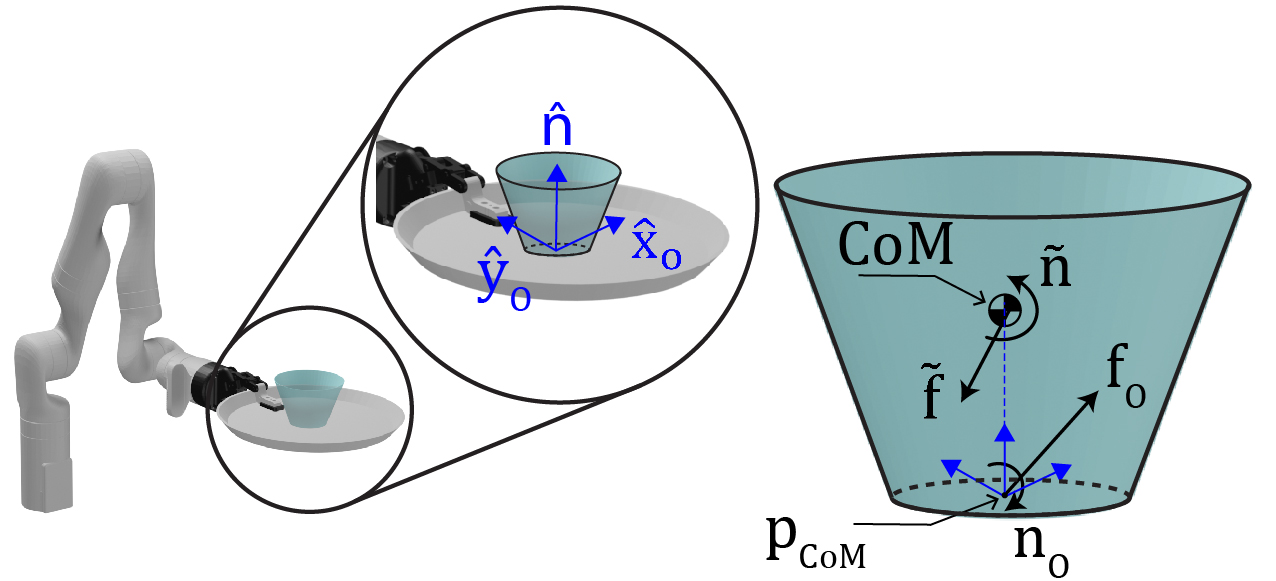}
    \caption{Free body diagram of a manipulated object. 
    $\COMproj{}$ is the projection of the COM onto the contact plane and $\eenorm$ is the contact normal vector and is chosen to be the z-axis of the contact frame.
    \vspace*{-0.5cm}}
    \label{fig:contactdiagram}
\end{figure}
    
    As illustrated in Fig. \ref{fig:contactdiagram}, the system considered in this paper is a robotic manipulator that grasps a flat surfaced tray with an object resting on it.
    The contact normal to the tray is denoted by the unit vector $\eenorm \in \R^3$.
    We make the following assumption:
    \begin{assum}
        \label{assum:contact}
        The manipulator maintains force closure on the tray for all time.
        The object has a circular contact area with radius $\objrad{}$.
        The friction between the tray and the object is described by linear Coulomb Friction where the coefficient of static friction is $\CoefFric$. 
    \end{assum}
    \noindent As a result, the tray can be treated as part of the link which maintains force closure on it, which we choose to be the $n_q^{\text{th}}$ link for convenience.
    The approach described in this paper can be generalized to ensure that a force closure condition was satisfied throughout the motion, but in the interest of simplicity, we leave that extension for future work.

\subsubsection{Kinematics and Dynamics}

Next, we introduce the kinematics and dynamics of the system.
Suppose there exists a fixed inertial reference frame, which we call the \defemph{world} frame, and a \defemph{base} frame, which we define as $0$\ts{th} frame, that indicates the origin of the robot's kinematic chain.
The $j$\ts{th} reference frame $\{\hat{x}_j, \hat{y}_j, \hat{z}_j\}$ is attached to the robot's $j$\ts{th} revolute joint, and $\hat{z}_j = [0, 0, 1]\trans$ corresponds to the $j$\ts{th} joint's axis of rotation for $j \in \Nq$.
For a configuration at a particular time, $\q$, the position and orientation of frame $j$ with respect to frame $j-1$ can be expressed using homogeneous transformations \cite[Ch. 2]{spong2005textbook} 
with a a configuration-dependent rotation matrix, $R\jssmu (\qj)$, and  a fixed translation vector from frame $j-1$ to frame $j$, $p\jssmu$.

The robot is composed of $\nq$ rigid links with inertial parameters given by the vector $ \Delta_{arm} = (m_1, \ldots, m_{\nq}, c_{x,1}, \ldots, c_{z,{\nq}}, I_{xx,1}, \ldots, I_{zz,{\nq}})\trans$, 
where $m_j$, $c_j =(c_{x,j}, c_{y,j}, c_{z,j})\trans$, and $(I_{xx,j}, \ldots, I_{zz,j})\trans$  represent the mass, center of mass (CoM), and inertia tensor of the $j$\ts{th} link, respectively. 
The properties for $\nq^{th}$ link are comprised of the gripping link and tray properties.
The object is a rigid body with inertial parameters given by the vector $ \Delta_{o} = (m_o, c_{x,o}, \ldots, c_{z,o}, I_{xx,o}, \ldots, I_{zz,o})\trans$,
where $m_o$, $c_o =(c_{x,o}, c_{y,o}, c_{z,o})\trans$, and $(I_{xx,o}, \ldots, I_{zz,o})\trans$  represent the mass, center of mass, and inertia tensor of the object, respectively.

The arm dynamics are represented by the standard manipulator equations \cite[(10.13)]{spong2005textbook}
\begin{equation}\label{eq:manipulator_equation}
    \begin{split}
    \hspace*{-0.25cm}
    \bM(\q, \Delta_{arm}) \qdd +
    \bC(\q, \qd, \Delta_{arm}) \qd +
    \bG(\q, \Delta_{arm}) \\
    =  u(t) + J\trans(\q) \cwrenchObj(\q)
    \end{split}
\end{equation}
where $\Mq \in \R^{\nq \times \nq}$ is the positive definite inertia matrix, $\Cqd \in \R^{\nq \times \nq}$ is the Coriolis matrix, $\Gqd \in \R^{\nq}$ is the gravity vector, $u(t) \in \R^{\nq}$ is the input torque, $J(\q) \in \R^{6 \times \nq}$ is the manipulator Jacobian, and $\cwrenchObj(\q) \in \R^{6}$ is the wrench applied by the object onto the manipulator, all at time $t$. 

Let $\cwrench = -\cwrenchObj$ be the wrench that the manipulator applies to the object.
We refer to $\cwrench$ as the contact wrench and define it as $\cwrench(\q) = 
\begin{bmatrix} 
\cforce(\q)\trans & \cmoment(\q)\trans \end{bmatrix}\trans$
where $\cforce(\q \in \R^3$ is the contact force and $\cmoment(\q) \in \R^3$ is the contact moment
 \cite[(5.99)]{Craig2005}.
If there are no other external wrenches applied to the object, then the object at time $t$ is manipulated solely by the contact wrench. 
The object dynamics can be represented using standard Newton-Euler equations for a rigid body.

While following a trajectory, if the object moves with respect to the tray during operation, we say that the object is experiencing \emph{$\relativeMotion$}.
{\bf This paper's goal is to devise a planning strategy to avoid $\relativeMotion$} when beginning from some initial condition while trying to reach a user-specified goal configuration or end effector position. 
Importantly, when there is no $\relativeMotion$, the object can be treated as a rigid body attached to the tray.
However, instead of treating it as part of the last link similar to the tray, we instead define a \emph{fixed joint} between the object and tray.
We define the location of the fixed joint to be at $\COMproj{}$, as shown in Fig. \ref{fig:contactdiagram}.
The reference frame for the fixed joint has the z-axis $\hat{z}_o$ chosen to be aligned with the contact normal $\eenorm$.
We choose the y-axis $\hat{y}_o$ to be oriented towards the plane of the supporting surface  aligned with joint $\nq$ and  
Throughout the remainder of the paper, we refer to this manipulator, tray, and object system as the \emph{extended arm}.
We also let $\Ne = \{1,\ldots,\nq+1\}$.

There are several benefits to this modeling choice.
The first is that the dynamics of the entire system, including the object, can be calculated using the forward pass of a modified Recursive Newton-Euler Algorithm (RNEA) which is described in App. B\footnote{Supplementary Appendices found at https://roahmlab.github.io/waitr-dev/}.
The second is that the backward pass of RNEA calculates the wrench exerted through the fixed joint by the tray on the object, which due to the choice of joint location, corresponds to the contact wrench applied to the object.

For the extended arm, the object inertial parameters in $\Delta_{o}$ are inserted into the $n_{q+1}^{th}$ positions in the manipulator inertial parameter vector $\Delta_{arm}$ to form the extended arm inertial parameter vector, $\Delta$.
We make the following assumption about the extended arm's inertial parameters:
\begin{assum}
\label{assum:arm_model}
The model structure (i.e., number of joints, sizes of links, etc.) of the robot is known, but its true inertial parameters $\trueparams$ are unknown. 
The uncertainty in each inertial parameter is known and given by the interval vector, $\intparams = \left(\iv{m_1}, \ldots, \iv{m_o}, \iv{c_{x,1}}, \ldots, \iv{c_{z,o}}, \iv{I_{xx,1}}, \ldots \iv{I_{zz,o}}\right)\trans$.
The true parameters lie in this interval, i.e., $\trueparams \in \intparams$.
All elements of this interval vector have bounded elements and any inertia tensor drawn from $[I_j]$ must be positive semidefinite.
\end{assum}
\noindent Before continuing, we make one final observation.
{\bf Note that the state of the robot depends on the inertia parameter vector.} 
We leave out this dependence for convenience unless we want to emphasize the importance of this dependence in which case we write $q(t;\Delta)$. 


\subsection{Trajectory Parameterization and Online Control}
\label{subsec:traj_parameterization}

\MethodName{} performs planning in a receding horizon fashion.
We assume without loss of generality that the control input and trajectory of a planning iteration begin at time $t=0$ and end at a fixed time $\tfin$.
To ensure persistent real-time operation, we require that \MethodName{} identifies a new trajectory parameter within a fixed planning time of $t\plan$ seconds, where $t\plan < \tfin$.
Thus \MethodName{} has limited planning time and must select a new trajectory parameter before completing its tracking of the previously identified desired trajectory.

In each planning iteration, \MethodName{} chooses a desired trajectory to be followed by the arm. 
These trajectories are chosen from a continuum of trajectories, with each uniquely determined by a \textit{trajectory parameter} $k \in K$ and are written as $q_d(t;k)$ or $q_d(t;k,\Delta)$ when we want to emphasize the dependence on the inertia parameter vector. 
The set $K \subset \R^{n_k}$, $n_k \in \N$, is compact and represents a user-designed continuum of trajectories. 
In general, $K$ can be designed to include trajectories designed for a wide variety of tasks and robot morphologies \cite{holmes2020armtd, kousik2020bridging, kousik2019_quad_RTD}.
Finally, we assume that $\ddot{q}_d( \cdot ;k)$ is a Lipschitz continuously differentiable function and at $\dot{q}_d( \tfin; k) = \ddot{q}_d( \tfin;k) = 0$.

A user can track a particular desired trajectory by applying some feedback controller.
As a result we write $q( \cdot ;k)$ for a trajectory of the system at time $k$.
Note \MethodName{} uses a specific type of feedback controller that is described in Sec. \ref{sec:controller}.
Because the feedback controller can be a function of the state and the desired trajectory and its derivatives, we define a \defemph{total feedback trajectory}, $\qA$, as:
\begin{equation} \label{eq:total_feedback_traj}
\qA(t;k) =
    \begin{bmatrix}
        q(t;k) & \dot{q}(t;k) &  q_d(t;k) & \dot{q}_d(t;k) & \ddot{q}_d(t;k)
    \end{bmatrix}.
\end{equation}

\subsection{Contact Constraints}
\label{subsec:contact_constraint}

To ensure that the object does not experience $\relativeMotion$, we need to compute the wrench applied on the object during motion.
This contact wrench is applied by the manipulator through the supporting surface, and an equivalent contact wrench $\cwrench$ can be written with respect to point $\COMproj{}$.
The contact wrench is a function of the robot's configuration trajectory and its velocity and acceleration.
\noindent Because there are no external wrenches applied to the object, the object at time $t$ is manipulated solely by the following contact wrench $\cwrench$. 
Then, the motion of the object can be constrained in all six degrees of freedom by using components of the contact wrench $\cwrench$.
Using the following lemma, whose proof can be found in App. C\footnotemark[1], one can ensure no $\relativeMotion$ occurs:
\begin{lem}
    \label{lem:constraint_satisfaction}
    Let the \emph{vertical separation} be defined as:
    \begin{equation}
    \label{eq:vsep}
        \sep(\cwrench(\qA(t;k))) := -\czforce(\qA(t;k))
    \end{equation}
    Let the \emph{linear slip} be defined as:
    \begin{equation}
    \label{eq:linslip}
    \begin{split}
         \slip(\cwrench(\qA(t;k)))& :=   (\cxforce(\qA(t;k)))^2 + \\ &+ (\cyforce(\qA(t;k)))^2 - (\CoefFric \czforce(\qA(t;k)))^2
    \end{split}
    \end{equation}
    Let the \emph{tip} of the object be defined as:
    \begin{equation}
        \label{eq:tipConstraint}
        \begin{split}
        \tip(\cwrench(\qA(t;k))) := (\eenorm{} & \times \cmoment(\qA(t;k)))^2 + \\ &- \objrad^2 (\eenorm{} \cdot \cforce(\qA(t;k)))^2
        \end{split}
    \end{equation}
    If  \eqref{eq:vsep},\eqref{eq:linslip}, and \eqref{eq:tipConstraint} are less than or equal to zero
      for all $t \in T$, then the object does not experience $\relativeMotion$ while the robot is moving along $q$.
\end{lem}

\subsection{Online Trajectory Optimization}

Then by using this parameterization and the constraints described in Sec. \ref{subsec:contact_constraint}, one could compute a trajectory to manipulate an unsecured object with zero relative motion by solving the following problem:
\begin{align}
    \label{eq:optcost}
    &\underset{k \in K}{\min} &&\texttt{cost}(k) \\
    \label{eq:optsepcon}
    &&&  \sep(\cwrench(\qA(t;k,\Delta))) \leq 0 &&\forall t \in T,  \Delta \in [\Delta] \\
    \label{eq:optslipcon}
    &&& \slip(\cwrench(\qA(t;k,\Delta))) \leq 0 &&\forall t \in T,  \Delta \in [\Delta]  \\
    \label{eq:opttipcon}
    &&& \tip(\cwrench(\qA(t;k,\Delta))) \le 0 &&\forall t \in T,  \Delta \in [\Delta].
\end{align}
Note that the total trajectory in this instance is a function of the parameterized trajectory and the true inertia parameters of the extended arm.
Because these parameters may not be known, we require that the constraints be satisfied for all possible inertia parameters.
Implementing a real-time optimization algorithm to solve this problem is challenging for several reasons.
First, the dynamics of the robot are nonlinear and constructing an explicit solution to them is intractable. 
Second, the constraints must be satisfied for all time $t$ in an uncountable set $T$.
\vspace*{-0.5cm}
\section{Planning Algorithm Formulation}
\label{sec:formulation}

The key technical idea of this work is forming polynomial zonotope overapproximations of the desired trajectories and the contact wrench.
This enables us to generate polynomial zonotope overapproximations of all constraints in the optimization problem \eqref{eq:optcost}-\eqref{eq:opttipcon} at the start of the planning iteration, then use these constraint overapproximations to perform optimization within $t_p$.
Then, we can use the robust passivity-based controller proposed in \cite[Sec. VII]{michaux2023} to bound the tracking error of the extended arm.
This section begins by summarizing the important properties of the aforementioned robust passivity-based controller. 
Subsequently, it describes how to represent the aforementioned constraints conservatively using polynomial zonotopes. 
This section concludes by describing how to transform the optimization problem \eqref{eq:optcost}-\eqref{eq:opttipcon} into an implementable version whose solutions can be be followed by the robot without experiencing $\relativeMotion$ while satisfying input constraints. 

\vspace*{-0.15cm}
\subsection{Robust Passivity-Based Controller}
\label{sec:controller}


We begin by restating Cor. 12 from \cite{michaux2023}:
\begin{lem}
\label{lem:tracking_bounds}
Suppose there exists a $\sigma_m > 0$ such that $0 < \sigm \leq \lambdamin(\bM(q(t),\trueparams))$  for all $q(t) \in Q$ and $\trueparams \in \intparams$.
Let $\err(t;k) := q_d(t;k) - q(t;k)$.
Let $e(0) = \dot{e}(0) = 0$, and let $\roblevel > 0$ and $K_m \geq 0$ be user-specified constants.
Define
\begin{equation}
\label{eq:ultimate_bound}
\pbound \coloneqq \frac{1}{K_{r}} \ultbound \quad  \regtext{and} \quad \vbound \coloneqq 2\ultbound. 
\end{equation}
Then there exists a feedback controller 
that when applied to the extended arm dynamics \eqref{eq:manipulator_equation} ensure that $|\errj(t)| \leq \pbound$ and $|\errjdot(t)| \leq \vbound$ for all $t \in T$ and $j \in \Ne$.
\end{lem}
\noindent We make two important remarks about the feedback controller. 
First, 
 we compute $\sigm$ numerically and verify that it is greater than zero during experiments.
Second, note that one can compute the feedback controller by applying a variant of the Recursive Newton Euler Algorithm (RNEA) \cite[Sec. VII-B]{michaux2023}. 
To bound the error for all time rather than just over $t \in T$ one can apply the previous lemma and \cite[Rem. 13]{michaux2023} which we summarize here for convenience:
\begin{rem}
\label{rem:combining_traj}
Suppose that the initial condition of the desired trajectory in the first planning iteration matches the actual state of the robot, and that the initial condition of all subsequent desired trajectories match the state of the previous desired trajectory at $t = t\plan$.
Then, one can satisfy the bounds described by Lem. \ref{lem:tracking_bounds}.
\end{rem}
\noindent Note that to apply Rem. \ref{rem:combining_traj}, one only needs to know the final state of the previous \defemph{desired trajectory} rather than the final state of the robot's \defemph{actual trajectory}.

\subsection{Polynomial Zonotope Overapproximation}

This subsection describes how \MethodName{} uses Lem. \ref{lem:tracking_bounds} to overapproximate parameterized desired trajectories to conservatively account for continuous time operation and tracking error within our optimization framework.
The desired trajectories $\qdesk$ are functions of only time $t$ and the trajectory parameter $k$.
As illustrated in Fig. \ref{fig:PZoverview}, our approach creates polynomial zonotope versions of $T$ and $K$, which are then plugged into the formula for $\qdesk$ to create polynomial zonotope overapproximations which are then extended to the joint trajectory, the force trajectory, and subsequently the tipping constraint for the manipulation problem.

\begin{figure}
    \centering
    \includegraphics[width=1.0\columnwidth]{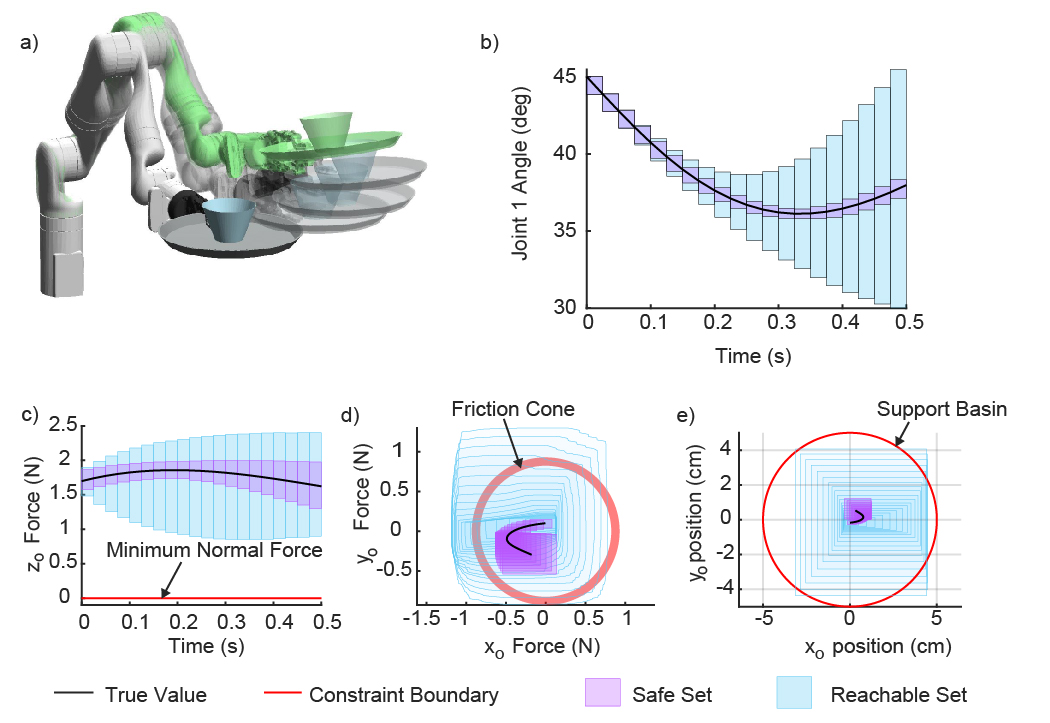}
    \caption{An illustration of the polynomial zonotope overapproximations of the configuration, force, and constraint trajectories.
    The motion of the robot is shown in intermediate configurations the top of the figure. 
    Blue indicates unsliced reachable sets, purple indicates sliced reachable sets and black lines are the nominal values.
    Red indicates a constraint boundary.
    \vspace*{-0.5cm}
    }
    \label{fig:PZoverview}
\end{figure}

\subsubsection{Time Horizon and Trajectory Parameter PZs}
We first create polynomial zonotopes representing the time horizon $T$.
Choose a timestep $\timestep$ so that $\nt \coloneqq \frac{T}{\timestep} \in \N$.
Let $N_t := \{1,\ldots,\nt\}$.
Divide the compact time horizon $T \subset \R$ into $\nt$ time subintervals.
We represent the $i$\ts{th} time subinterval corresponding to $t \in \iv{(i-1)\timestep, i\timestep}$ as a polynomial zonotope $\pz{T_i} = \left\{t \mid 
            t = \tfrac{(i-1) + i}{2}\timestep + \tfrac{1}{2} \timestep \tvari,\ \tvari \in [-1,1]
        \right\}$.
with indeterminate $\tvari \in \iv{-1, 1}$.

In this work, we choose $K = \bigtimes_{j=1}^{n_q} K_j$, where each $\Kj$ is a compact one-dimensional interval.
For simplicity, let each $\Kj = \iv{-1, 1}$.
We represent the interval $\Kj$ as a polynomial zonotope $\pz{\Kj} = \kjvar$ where $\kjvar \in \iv{-1, 1}$ is an indeterminate.
Let $\kvar \in \iv{-1,1}^{n_q}$ denote the vector of indeterminates where the $j$\ts{th} component of $\kvar$ is $\kjvar$.
With this choice of $\Kj$, any particular $k_j$ directly yields $k_j = \setop{slice}(\pz{\Kj}, \kjvar, k_j)$ (see (A14)\footnote{Supplementary Appendices found at https://roahmlab.github.io/waitr-dev/}).

\subsubsection{Trajectory PZs}
\label{sec:traj_pz}

Recall that $\pz{T_i}$ and $\pz{K_j}$ have indeterminates $\tvari$ and $\kjvar$, respectively.
Because the desired trajectories only depend on $t$ and $k$, the polynomial zonotopes $\pzqdesji$, $\pzqddesji$ and $\pzqdddesji$ depend only on the indeterminates $\tvari$ and $\kvar$.
By plugging in a given $k$ for $\kvar$ via the $\setop{slice}$ operation, we obtain a polynomial zonotope where $\tvari$ is the only remaining indeterminate. 
Because we perform this particular slicing operation repeatedly throughout this document, if we are given a polynomial zonotope, $\pzqdesji$, we use the shorthand $\pzqdesjki = \setop{slice}(\pzqdesji, \kvar, k)$.
Because of our previous observation, $\qdeskj \in \pzqdesjki$ for all $t \in \pz{T_i}$.

Next, we use Lem. \ref{lem:tracking_bounds} to generate polynomial zonotopes that overapproximate any trajectory that is followed by the robot.
\begin{lem}
\label{lem:pz_q_bound}
Let $\pzpboundj = \pboundj\epvarj$ and $\pzvboundj = \vbound\evvarj$, with indeterminates $\epvarj \in \iv{-1, 1}$ and $\evvarj \in \iv{-1, 1}$.
Then, let 
\begin{align}
    \pzqji &:= \pzqdesji \oplus \pzpboundj \\
    \pzqdji &:= \pzqddesji \oplus \pzvboundj 
\end{align}
for each $i \in \Nt$.
The polynomial zonotopes $\pzqji$ and $\pzqdji$ overapproximate the set of all joint trajectories that can be executed by the robot, i.e., for each $k \in \pz{K}$ and $j \in \Ne$, $q_j(t; k) \in \pzqjki \text{   and   } \dot{q}_j(t; k) \in \pzqdjki, \forall t \in \pz{T_i}$.
\end{lem}
\noindent Moving forward, let $\pzqdesi = \bigtimes_{j=1}^{n_q} \pzi{q_{d, j}}$ and $\pzqi = \bigtimes_{j=1}^{n_q} \pzi{q_{j}}$
We similarly define $\pzqddesi$, $\pzqdi$, $\pzqdddesi$, and $\pzqddi$.
Furthermore, let $\pz{\pboundvec} = \bigtimes_{j=1}^{n_q}\bm{\epsilon}_\mathbf{p, j}$  and $\pz{\vboundvec} = \bigtimes_{j=1}^{n_q}\bm{\epsilon}_\mathbf{v, j}$.

After generating the polynomial zonotope overapproximations of joint trajectories in the previous subsection, \MethodName{} composes these together to compute a polynomial zonotope overapproximation of the contact wrench by applying Alg. \ref{alg:PZRNEA}.
This algorithm is based on RNEA, which is a tool to compute the dynamics of serial chain manipulators.
Alg. \ref{alg:PZRNEA} is a polynomial zonotope version of the RNEA, and was originally proposed in \cite{michaux2023}.
Importantly, because of the extended arm modeling choice described in Sec. \ref{sec:modeling}, this same algorithm now computes an overapproximation of the contact wrench.
Using \cite[Lem. 7]{michaux2023} and the fact that the vertical separation, linear slip, and tip constraints are all polynomial functions, one can prove the following result:
\vspace*{-0.1cm}
\begin{lem}\label{lem:PZSlipConstraints}
Let $\{\pz{w\jss}(\pzqAi, \intparams)\}_{j \in \Ne} =$ $~\texttt{PZRNEA}(\pzqAi, \intparams)$ for each $i \in \Nt$.
Then the vertical separation, linear slip, and tip constraints can be overapproximated by polynomial zonotopes using $\pz{w^{n_q+1}_{n_q+1}}(\pzqAi, \intparams)$, i.e., for each $k \in \pz{K}$
\begin{align}
    \sep(\cwrench(\qA(t;k,\Delta))) &\in \pz{\sep}(\pz{w^{n_q+1}_{n_q+1}}(\pzqAki, \intparams)) \\
    \slip(\cwrench(\qA(t;k,\Delta))) &\in \pz{\slip}(\pz{w^{n_q+1}_{n_q+1}}(\pzqAki, \intparams)) \\
    \tip(\cwrench(\qA(t;k,\Delta))) &\in \pz{\tip}(\pz{w^{n_q+1}_{n_q+1}}(\pzqAki, \intparams)),
\end{align}
$\forall t \in \pz{T_i}$ and $\Delta \in \intparams$.
\end{lem}



\begin{algorithm}[t]
\small
\begin{algorithmic}[1]
\State Let $\pzqdai := \pzqdesi \oplus K_r \pz{\pboundvec}$.
\State Let $\pzqddai := \pzqddesi \oplus K_r \pz{\vboundvec}$.
\State Let $a_0^0 =  \begin{bmatrix} 0 & 0 & 9.81 \end{bmatrix}^\top$.
\State{\bf parfor} $i \in \Nt$ 

\State\hspace{0.2in}{\bf for} $j = 1:\nq+1$ 
    
    \State\hspace{0.4in}$\pz{R\jssmu}, \pz{\nom{p\jssmu}} \leftarrow \pz{\homtrans_{j}^{j-1}}(\pzqji)$ 
    
    \State\hspace{0.4in} $\pz{R\jssm} \gets \texttt{pzTranspose}(\pz{R\jssmu})$
    
\State\hspace{0.2in}{\bf end for}

\State\hspace{0.2in}{\bf for} $j = 1:\nq+1$ 

    \State\hspace{0.4in} $\pz{\omega\jss} \gets \pz{R\jssm} \pz{\omega\jssmm} \bigoplus \pz{\qdjnot}\zj$
    
    \State\hspace{0.4in} $\pz{\omega\jssa} \gets \pz{R\jssm} \pz{\omega\jssmma} \bigoplus \pz{\qajdotnot}\zj$
    
    \State\hspace{0.4in} $\pz{\dot{\omega}\jss} \gets \pz{R\jssm} \pz{\dot{\omega}\jssmm} \bigoplus ((\pz{R\jssm} \pz{\omega\jssa}) \bigotimes (\pz{\qdjnot} \zj)) \bigoplus \pz{\qajddotnot} \zj$ 
    
    
    \State\hspace{0.4in} $\pz{a\jss} \gets (\pz{R\jssm} \pz{a\jssmm}) \bigoplus (\pz{\dot{\omega}\jss} \bigotimes \pz{p\jssmu}) \bigoplus (\pz{\omega\jss} \bigotimes (\pz{\omega\jssa} \bigotimes \pz{p\jssmu}))$

    \State\hspace{0.4in} $\pz{a\jssCOM} \gets \pz{a\jss} \bigoplus (\pz{\dot{\omega}\jss} \bigotimes \pz{p\jssCOM}) \bigoplus (\pz{\omega\jss} \bigotimes (\pz{\omega\jssa} \bigotimes \pz{p\jssCOM})) $
    
    \State\hspace{0.4in} $\pz{F\jss} \gets \pz{m_i} \pz{a\jssCOM}$
    
    \State\hspace{0.4in} $\pz{N\jss} \gets \pz{I\jss} \pz{\dot{\omega}\jss} \bigoplus (\pz{\omega\jssa}  \bigotimes (\pz{I\jss} \pz{\omega\jss})) $
    
\State\hspace{0.2in}{\bf end for}

\State\hspace{0.2in}Initialize $\pz{f^{\nq+2}_{\nq+2}}, \pz{n^{\nq+2}_{\nq+2}}, \pz{R^{\nq}_{\nq+2}}$

\State\hspace{0.2in}{\bf for} $j = \nq+1:-1:1$ 

\State\hspace{0.4in} $\pz{f\jss}(\pzqAi, \intparams) \gets (\pz{R\jssp} \pz{f^{j+1}_{j+1}}) \bigoplus \pz{F\jss} $

\State\hspace{0.4in} $ \pz{n\jss}(\pzqAi, \intparams) \gets (\pz{R\jssp} \pz{n^{j+1}_{j+1}}) \bigoplus (\pz{p\jssCOM} \bigotimes \pz{F\jss}) \newline \hspace*{7em} \bigoplus (\pz{p\jssp} \bigotimes (\pz{R\jssp} \pz{f^{j+1}_{j+1}})) \bigoplus \pz{N\jss} $ 

\State\hspace{0.4in} $ \pz{w\jss}(\pzqAi,\intparams) \gets \begin{bmatrix}\pz{f\jss}(\pzqAi, \intparams) \\ \pz{n\jss}(\pzqAi, \intparams)  \end{bmatrix} $ 





\State\hspace{0.2in}{\bf end for}

\State{\bf end parfor} 
\end{algorithmic}
\caption{\small  $\{\pz{w\jss}(\pzqAi, \intparams)\}_{j \in \Ne} = ~\texttt{PZRNEA}(\pzqAi, \intparams)$}
\label{alg:PZRNEA}
\end{algorithm}

\vspace*{-0.45cm}
\subsection{Implementable Online Optimization Problem}
\label{sec:PZOptProb}
\vspace*{-0.10cm}

Using the polynomial zonotope overapproximations of the vertical separation, linear slip, and tip constraints, one can construct an representation of \eqref{eq:optcost}-\eqref{eq:opttipcon}:
\vspace*{-0.15cm}
\begin{align}
    \label{eq:pz_optcost}
    &\underset{k \in \pz{K}}{\min} &&\texttt{cost}(k) \\
    \label{eq:pz_optsepcon}
    &&& \hspace*{-0.25cm} \setop{sup}(\pz{\sep}(\pz{w^{n_q+1}_{n_q+1}}(\pzqAki, \intparams)) \leq 0 \quad &\forall i \in \N_t \\
    \label{eq:pz_optslipcon}
    &&&  \hspace*{-0.25cm} \setop{sup}(\pz{\slip}(\pz{w^{n_q+1}_{n_q+1}}(\pzqAki, \intparams)) \leq 0 \quad & \forall i \in \N_t \\
    \label{eq:pz_opttipcon}
    &&&  \hspace*{-0.25cm} \setop{sup}(\pz{\tip}(\pz{w^{n_q+1}_{n_q+1}}(\pzqAki, \intparams)) \leq 0 \quad & \forall i \in \N_t.
\end{align}
Using Lem. \ref{lem:PZSlipConstraints} one can prove the following theorem about this optimization problem.
\begin{thm}
\label{thm:optProbSafe}
    Any feasible solution $k$ to the optimization problem described in \eqref{eq:pz_optcost}-\eqref{eq:pz_opttipcon} parameterizes a trajectory that results in no $\relativeMotion$ between an unsecured object and the end-effector of the robot over the time horizon $T$.
\end{thm}

The implementation of this optimization problem is summarized in Alg. \ref{alg:opt}.
Note in particular that we apply Rem. \ref{rem:combining_traj} and Lem. \ref{lem:pz_q_bound} to compute the polynomial zonotope overapproximations of the parameterized trajectories and the actual trajectory of the system for each time interval $i \in \Nt$ in parallel in Line \ref{lin:traj}.
In addition, note that one can calculate the analytical gradients of the cost function and constraints and provide them to an optimization solver.
One can use polynomial differentiation of the polynomial zonotope representations of the constraints to compute the required gradients \cite[Sec. IX-C]{michaux2023}.
In conjunction with \TROtitle{}, additional constraints can be added to the optimization problem such that the parameterized trajectory is guaranteed collision-free and can be tracked while satisfying joint and input limits \cite[Sec. IX-B]{michaux2023}.
Thus, a solution to this optimization problem results in a safe manipulation trajectory.

\begin{algorithm}[t]
\small
\begin{algorithmic}[1]
    
    
    
    


    \State{\bf Parfor} $i = 1:\Nt$ 


    \State\hspace{0.1in} $\{ \pzqi, \ldots, \pzqdddesi  \} \leftarrow \setop{PZ}(\initq,\initdq,\initddq)$ // Sec. \ref{sec:traj_pz} \label{lin:traj}

    
    \State\hspace{0.1in} // object wrench reachable set using Alg. \ref{alg:PZRNEA} //

    \State\hspace{0.1in} $\{\pz{w\jss}(\pzqAi, \intparams)\}_{j \in \Ne} = ~\texttt{PZRNEA}(\pzqAi, \intparams)$ 

    \State\hspace{0.1in} // constraints using Lem. \ref{lem:PZSlipConstraints} and $\pz{w^{n_q+1}_{n_q+1}}(\pzqAi, \intparams)$ //
    
    \State\hspace{0.1in} {\bf Compute} $\pz{\sep}, \pz{\slip},$ and $\pz{\tip}$

    \State\hspace{0in}{\bf End Parfor}

    \State{\bf Try:} $k^* \gets$ solve \eqref{eq:pz_optcost} -- \eqref{eq:pz_opttipcon}
    \State{\bf Catch:} (if $t\elapsed > t\plan$), {\bf then}  $k^* = \texttt{NaN}$ // $t\elapsed$ measures the amount \phantom{dumb} of time since $\numop{Opt}$ was called //

    
    \end{algorithmic}
\caption{\small ${k^* = \numop{Opt}(\initq,\initdq,\initddq,\numop{cost},t\plan, \Nt, \Delta_0, [\Delta],\pbound,\vbound)}$}
\label{alg:opt}
\end{algorithm}

\vspace*{-0.1cm}
\subsection{\MethodName{}'s Online Operation}
\label{subsec:OnlinePlanning}

\begin{algorithm}[t]
\small
\begin{algorithmic}[1]
    \State {\bf Require:} $t\plan > 0$, $\Nt \in \N, [\Delta],\Delta_0 \in [\Delta],$ $\texttt{cost}: K \to \R,$, $\pbound$, $\vbound$
    \State {\bf Initialize:} $l = 0$, $t_j = 0$, and \newline
    \phantom{\bf Initialize:} $k^*_l = \texttt{Opt}(\qstart,\zeros,\zeros,\obsset,\texttt{cost}, t\plan, \Nt, \Delta_0, [\Delta],\pbound,\vbound)$
    \State {\bf If} $k^*_{l} = \texttt{NaN}$, {\bf then} execute brake 
    \State{\bf Loop:}

        \State \hspace{0.05in} // Line \ref{lin:apply} executes simultaneously with Lines \ref{lin:opt} -- \ref{lin:else} //

        \State \hspace{0.05in}  // Use Lem. \ref{lem:tracking_bounds} //
        \State \hspace{0.05in}  {\bf Apply} feedback controller  to robot for $t \in [t_l,t_l + t\plan]$ \label{lin:apply} 

        \State \hspace{0.05in}  $k^*_{l+1} = \texttt{Opt}(q_d(t\plan;k^*_l), \dot{q}^{j}_d(t\plan;k^*_l), \ddot{q}^{j}_d(t\plan;k^*_l),\texttt{cost},t\plan, \Nt, \Delta_0,$ \newline \phantom{stupid stupid } $[\Delta], \pbound, \vbound)$ \label{lin:opt}
        \State \hspace{0.05in} {\bf If} \label{lin:no_k} $k^*_{l+1} = \texttt{NaN}$, {\bf then} execute brake and break loop
        \State \hspace{0.05in} {\bf Else}  $t_{l+1} \leftarrow t_l + t\plan$ and $j \leftarrow j + 1$ \label{lin:else}

    \State{\bf End}
\end{algorithmic}
\caption{\small \MethodName{} Online Planning and Control}
\label{alg:armtdForcePlan}
\end{algorithm}

The online operation of \MethodName{} is summarized in Alg. \ref{alg:armtdForcePlan}.
Because \MethodName{} operates in receding horizon fashion, the controller in Lem. \ref{lem:tracking_bounds} is used to track the trajectory parameter computed during the previous planning iteration on Line \ref{lin:apply}. 
At the same time, \MethodName{} computes the trajectory parameter for the following planning iteration on Line \ref{lin:opt}.
Recall that if a new trajectory parameter is not found before time $t_{p}$, then the robust controller tracks the braking maneuver of the previous trajectory to bring the robot to a safe stop, seen on Line \ref{lin:no_k}.
Further, by applying Lem. \ref{lem:tracking_bounds} and Thm. \ref{thm:optProbSafe}, one can prove that \MethodName{} generates behavior that is dynamically feasible and results in no $\relativeMotion$.

\section{Experiments}
\label{sec:experimental}

This section details the implementation and testing of the \MethodName{} framework on a Kinova Gen3 7 DOF robotic arm.
\MethodName{} is implemented in C++ and CUDA and a MEX version is used in MATLAB for simulation experiments. 
The code can be found at \url{https://github.com/roahmlab/waitr-dev}.
Simulation experiments were run using a AMD Ryzen 9 5950x processor and an NVIDIA RTX A6000 GPU.
Two hardware experiments were run using an AMD Ryzen 9 3950x processor and NVIDIA Quadro RTX 8000 GPU.


\subsection{Trajectory Creation}
\label{sec:pzTrajImplementation}

	This work parameterizes the reference trajectories by using a set of degree 5 Bernstein polynomials as is done in \cite[Sec. IX-A]{michaux2023}.
	Letting $T=[t_{0},t_{f}]$, WLOG, the reference trajectories take the form $ \qdeskj = \sum_{l= 0}^{5} \beta_{j, l} b_{j, l}(t)$,
	where $\beta_{j, l} \in \R$ are the Bernstein Coefficients and $b_{j, l}: T \to \R$ are the Bernstein Basis Polynomials of degree 5 
    for each $l \in \{0, \ldots, 5\}$.    
    As a result of Rem. \ref{rem:combining_traj}, the initial position, velocity, and acceleration of a parameterized trajectory are constrained. 
	Further, the parameterized trajectories are constrained to have zero velocity and acceleration by $\tfin$ in order to bring the robot to a stop.
	Therefore, one can show that five of the six Bernstein coefficients $\beta_{j, \nu}$ are known and only the last coefficient can be optimized over. 
    We let the last coefficient $\beta_{j, 5} = \kjscale \kj + \kjoffset$ 
    where $\kjscale$ and $\kjoffset \in \R$ are user-specified constants.
    The choice of trajectory parameter $k_j$ directly determines the coefficient $\beta_{j, 5}$ and the final position of the robot.

\subsection{Implementation Details}

    \subsubsection{Robot Model and Environment}
    \label{subsec:experimental_model}
                
        For simulation, the tray mass was $0.044$ kg and the object mass was $0.172$ kg.
        For the first hardware experiment, the tray mass was $0.058$ kg and the object mass was $0.048$ kg.
        For the second hardware experiment, the tray mass was $0.354$ kg and the object mass was $0.592$ kg.
        These are added to the robot model as described in Sec. \ref{sec:modeling}.
        By sampling, the minimum eigenvalue of the mass matrix was determined to be uniformly bounded from below by $\sigma_m = 8.0386$. 
        
        
    \subsubsection{Trajectories}

        In all experiments, $t_{p} = 1.0$s and $\tfin = 2.0$s.
        In simulation experiments, we let $\kjoffset = \initqj$ and $\kjscale = \frac{\pi}{72}$, so that after $\tfin = 2.0$s the final position of any joint trajectory can differ from its initial position by up to $\pm\frac{\pi}{72}$ radians.
        For hardware experiments, we let $\kjoffset = \initqj$ and $\kjscale = \frac{\pi}{32}$ for $j \in {1,2,6}$ and $\kjscale = \frac{\pi}{72}$ for $j \in {3,4,5,7}$. 
        
    \subsubsection{Tracking Error Bound}
        
        We use the same controller as presented in \cite[Sec. VII]{michaux2023}, with a different $K_{r}$, $\roblevel$ and $\sigma_m$, which are reported in App. D\footnote{Supplementary Appendices found at https://roahmlab.github.io/waitr-dev/}.
    

    \subsubsection{High-level Planners}
    
        We use a cost function that minimizes the distance between $q(t_{p};k)$ and an intermediate waypoint.
        Simulation experiments are run with two different high level planners.
        The first is a straight line high-level planner (SL-HLP), which calculates a waypoint along a straight line between the start and goal in configuration space without checking for collision.
        The second is a graph-based high level planner (GB-HLP), which is constructed using robot configurations with a flat end-effector pose.
        Before operation, the configurations that are in collision with obstacles are removed from the graph. 
        Note that the high-level planner does not need safety guarantees because \MethodName{}'s safety guarantees are independent of the high-level planner.
        
    \subsubsection{Comparison Framework}

        We compare \MethodName{} to a previous method \TROtitle{}, which does not contain contact wrench constraints, by running both methods on identical simulation worlds with the same high level planner. 
        
        
    \subsubsection{Trajectory Optimization Implementation}
        
        The \MethodName{} C++ framework uses IPOPT \cite{Wchter2006OnTI} to solve the trajectory optimization problem during each planning iteration.
        Analytic gradients/subgradients of the cost function and constraints are provided for the optimization solver to speed computation.

\vspace*{-0.15cm}
\subsection{Simulation Experiments}

    \subsubsection{Simulation Setup}
        The simulation experiment consisted of testing \MethodName{} with two different high-level planners in 100 trials where each was initialized with a random feasible start and goal state.
        Each trial had 10 box shaped obstacles randomly placed in the workspace and randomly scaled side lengths that were allowed to vary between $0.010$ to $0.050$ m.
        
    \subsubsection{Results}
    
        \begin{table}
            \centering
            \begin{tabular}{c|c c c}
                 Method & Goals Reached & Safe Stops & Violations \\
                 \hline
                 \MethodName{} + SL-HLP & 44 & 56 & {\bf 0} \\ 
                 \hline
                 \MethodName{} + GB-HLP & {\bf 91} & {\bf 9} & {\bf 0} \\ 
                 \hline
                 \TROtitle{} \cite{michaux2023} + GB-HLP & 64 & 30 & 6 \\ 
            \end{tabular}
            \caption{Results on 100 random simulation experiments. 
            \vspace*{-0.75cm}
            }
            \label{tab:expResults}
        \end{table}
    
        The results of the 100 random trials are presented in Tab. \ref{tab:expResults}, with example videos available at \url{https://roahmlab.github.io/waitr-dev/}.
        \MethodName{} robustly handles infeasible waypoints from the SL-HLP while still ensuring safety, at the cost of not always making it to the goal.
        When using the GB-HLP, which provides feasible waypoints, \MethodName{} reaches the goal 91 times and safely stops 9 times. 
        In contrast, even though \TROtitle{} uses the GB-HLP, and is given feasible waypoints, it reaches the goal only 64 times, safely stops 30 times, and violates the contact constraints 6 times. 
        This shows that it is not adequate to have a high-level planner that can provide feasible waypoints.

\subsection{Hardware Results}

    \subsubsection{Setup}
    The first experiment compares \MethodName{} and \TROtitle{} on a Kinova Gen3 robot arm.
    The tray surface is covered with Duct tape to adjust the coefficient of friction, which was experimentally measured to be $\mu=0.36$.
    Both frameworks used the GB-HLP, and had the same random noise added to the last three joints in the waypoints selected.
    This randomly adjusted the orientation of the tray in each waypoint such that it was not flat. 
    Two obstacles were placed randomly in the workspace but at the same location when running each approach.
    This ensures that the only difference between the experiments for the two frameworks is the constraints that they are enforcing.
    The second experiment demonstrates the ability of \MethodName{} to operate under the presence of modeling uncertainty.
    The object for this experiment was a plastic cup filled with metal nuts.
    The coefficient of friction between the plastic cup and the tray was measured to be $\mu=0.60$.
    \MethodName{} was run in three consecutive trials which each had random placed obstacles.
    Before the last trial, the mass of the cup was adjusted by removing $0.028$ kg, or $4.73\%$ of the total mass of the unsecured object.
    This also adjusted the inertia by $4.73\%$.
    \MethodName{} was set to account for $5\%$ variation in the mass and inertia of the object, and the model used by the \MethodName{} framework was not adjusted to reflect the changed object parameters.
    
    \subsubsection{Results}
    Videos of the hardware experiments are available at https://roahmlab.github.io/waitr-dev/.
    The first hardware experiment shows \MethodName{} successfully manipulating an unsecured object while navigating around two randomly placed obstacles.
    \TROtitle{} also successfully navigates around the obstacles, but fails to manipulate the unsecured object, resulting in the object slipping and falling off the tray.
    The second trial demonstrates that \MethodName{} can successfully operate in the presence of uncertainty in the inertial parameters of the object being manipulated.
    All of the hardware experiments were run in real-time. 
\section{Conclusion}
\label{sec:conclusion}

This paper describes \MethodName{}, a real-time provably safe planning and control framework for collision-free non-prehensile manipulation capable of robustly dealing with uncertainty in both the manipulator's and target object's inertial parameters. 
There are several directions of future work for the \MethodName{} framework.
The first is enabling \MethodName{} to handle changes in the contact state of the objects being manipulated, such as allowing manipulated objects to slip on the tray.
The second is to form constraints that can guarantee safe force closure on objects throughout a desired trajectory.


\renewcommand{\bibfont}{\normalfont\footnotesize}
{\renewcommand{\markboth}[2]{}
\printbibliography}


\end{document}